\documentclass[10pt,twocolumn,twoside]{IEEEtran}

\usepackage{spconf,amsmath,graphicx}

\usepackage{epsfig,psfrag}
\usepackage{pst-all}
\usepackage{amsmath,amsthm,amssymb,amsfonts,upref,cite,epsf,color,bm}
\usepackage{graphicx}
\usepackage{color}
\usepackage{amsmath}
\usepackage{graphicx}
\usepackage{calc}

\newtheorem{theorem}{Theorem}[section]

\newtheorem{proposition}[theorem]{Proposition}

\newtheorem{algorithm}{Algorithm}

\renewcommand{\vec}[1]{\mathbf{#1}}

\newcommand{\mat}[1]{\mathbf{#1}}
\newcommand{\set}[1]{\mathcal{#1}}
\newcommand{\timeidx}{n}
\newcommand{\lagidx}{m}
\newcommand{\lagvar}{\lagidx}

\newcommand{\ACF}{\mathbf{R}}
\newcommand{\SDM}{\mathbf{S}}
\newcommand{\invSDM}{\mathbf{K}}
\newcommand{\SDMhat}{\widehat{\mathbf{S}}}

\newcommand{\GLASSO}{\widehat{\mathbf{K}}}
\newcommand{\GLASSOERR}{\mathbf{\Delta}}
\newcommand\defeq{:=}
\newcommand{\Exp}{\operatorname{E}}

\newcommand{\diag}{\operatorname{diag}}

\newcommand{\freqbin}{f}
\newcommand{\sparsity}{s}
\newcommand{\ESDM}{\widehat{\mathbf{S}}}  
\newcommand{\optvar}{\mathbf{X}}
\newcommand{\secondvar}{\mathbf{Z}}
\newcommand{\dualvar}{\mathbf{U}}
\DeclareMathOperator*{\argmin}{argmin}
\DeclareMathOperator*{\vectorize}{vec}

\DeclareMathOperator{\gsupp}{gsupp}

\newcommand{\coefflen}{p}
\newcommand{\samplesize}{N}
\newcommand{\nrfreqs}{F}
\newcommand{\vx}{\mathbf{x}}  
\newcommand{\cig}{\mathcal{G}_x}  

\DeclareMathOperator*{\PR}{P}
\newtheorem{thm}{Theorem}[section]

\newtheorem{assumption}[thm]{Assumption}

\usepackage{hyperref}

\title{Graphical LASSO based Model Selection for Time Series}

\name{Alexander Jung, Gabor Hannak and Norbert G{\"o}rtz}

\address{\small Institute of Telecommunications, Vienna University of Technology\\[-1mm]
\small Gusshausstrasse 25/389, 1040 Vienna, Austria; 
\small email: \{ajung,ghannak,norbert.goertz\}@nt.tuwien.ac.at}

\begin{document}
%
\maketitle
\begin{abstract}

We propose a novel graphical model selection scheme 
for high-dimensional stationary time series or discrete time processes. The method is based on a natural generalization of 
the graphical LASSO algorithm, introduced originally for the case of i.i.d.\ samples, and estimates the conditional 
independence graph of a time series from a finite length observation. 
The graphical LASSO for time series is defined as the solution of an $\ell_{1}$-regularized maximum (approximate) likelihood problem. 
We solve this optimization problem using the alternating direction method of multipliers. 
Our approach is nonparametric as we do not assume a finite dimensional parametric model, but only require the process to be sufficiently smooth in the spectral domain. 
For Gaussian processes, we characterize the performance of our method theoretically by deriving an upper bound on 
the probability that our algorithm fails. Numerical experiments demonstrate 
the ability of our method to recover the correct conditional independence graph from a limited amount of samples. 

\end{abstract}

\vspace*{-2mm}
\section{Introduction}\label{sec:intro}
\vspace*{-1mm}
We consider the problem of inferring the conditional independence graph (CIG) of 
a stationary high-dimensional discrete time process or time series $\vx[\timeidx]$ from observing $\samplesize$ samples $\mathbf{x}[1],\ldots,\mathbf{x}[\samplesize]$. This problem is referred to as graphical model selection (GMS) 
and of great practical interest, e.g., for gene analysis where the process $\vx[\timeidx]$ represents measurements of gene expression levels and the 
CIG gives insight into the dependencies between different genes \cite{Gohlke2008,DavidsonLevin2005}.


A first nonparametric GMS method for high-dimensional time series has been proposed recently \cite{JuHeck2014}. 
This approach is based on performing neighborhood regression (cf.Ê\cite{MeinBuhl2006}) in the frequency domain. 

In this paper, we present an alternative nonparametric GMS scheme for time series based on 
generalizing the graphical LASSO (gLASSO) \cite{FriedHastieTibsh2008,BuhlGeerBook,WittenFriedSimon2011} to stationary time series. 
The resulting algorithm is implemented using the alternating direction method of multipliers (ADMM) \cite{DistrOptStatistLearningADMM}, 
for which we derive closed-form update rules. 

While algorithmically our approach is similar to 
the joint gLASSO proposed in \cite{DanaherGroupGLASSO2014}, the deployment of a gLASSO type algorithm for GMS of time series seems to be new. 

Our main analytical contribution is a performance analysis which yields an upper bound on the probability that our scheme fails to correctly identify the CIG. The effectiveness of our GMS method is also verified by means of numerical experiments.


\emph{Notation.} 
Given a natural number $\nrfreqs$, we define $[\nrfreqs] \defeq \{1,\ldots,\nrfreqs\}$. 
For a square matrix $\mathbf{X} \!\in\! \mathbb{C}^{\coefflen \times \coefflen}$, we denote by $\bar{\mathbf{X}}$, $\mathbf{X}^{H}$, $\rm{tr}\{ \mathbf{X} \}$ and $|\mathbf{X}|$ its 
elementwise complex conjugate, its Hermitian transpose, its trace and its determinant, respectively. We also need the matrix norm $\| \mathbf{X} \|_{\infty} \defeq \max_{i,j} |X_{i,j}|$. 
By writing $\mathbf{X} \succeq \mathbf{Y}$ we mean that $\mathbf{Y} - \mathbf{X}$ is a positive-semidefinite (psd) matrix. 

We denote by $\mathcal{H}_{\coefflen}^{[\nrfreqs]}$ the set of all length-$\nrfreqs$ sequences $\optvar[\cdot] \!\defeq\! \big(\optvar[1],\ldots,\optvar[\nrfreqs]\big)$ 
with Hermitian matrices $\optvar[\freqbin] \!\in\! \mathbb{C}^{\coefflen \times \coefflen}$. For a sequence $\optvar[\cdot] \!\in\! \mathcal{H}_{\coefflen}^{[\nrfreqs]}$, we define 
$\| X_{i,k}[\cdot] \|^2 \!\defeq\! (1/\nrfreqs)   \sum_{\freqbin \in \nrfreqs} |X_{i,j}[\freqbin]|^{2}$, 
its squared generalized Frobenius norm $\|  \optvar[\cdot] \|^2_{\rm{F}} \defeq \sum_{i,j} \| X_{i,k}[\cdot] \|^2$ and its $\ell_{1}$-norm as $\|\optvar\|_{1}\defeq \sum_{i,j} \| X_{i,j}[\cdot] \|$.
We equip the set $\mathcal{H}_{\coefflen}^{[\nrfreqs]}$ with the inner product 
$\langle \mathbf{A}[\cdot],\mathbf{B}[\cdot] \rangle \defeq (1/\nrfreqs) \sum_{\freqbin \in [\nrfreqs]} {\rm tr} \{ \mathbf{A}[\freqbin] \mathbf{B}[\freqbin] \}$.

For a sequence $\optvar[\cdot] \in \mathcal{H}_{\coefflen}^{[\nrfreqs]}$ and some subset $\mathcal{S} \subseteq [\coefflen] \times [\coefflen]$, we 
denote by $\mathbf{X}_{\mathcal{S}}[\cdot]$ the matrix sequence which is obtained by, separately for each index $\freqbin \!\in\! [\nrfreqs]$, zeroing all entries of the matrix $\optvar[\freqbin]$ except those in $\mathcal{S}$. The generalized support of a sequence $\optvar[\cdot] \in \mathcal{H}_{\coefflen}^{[\nrfreqs]}$ is defined as 
$\gsupp(\optvar[\cdot]) \defeq \{ (i,j) \in [\coefflen] \times [\coefflen]: \big( \optvar[\freqbin] \big)_{i,j}Ê\neq 0 \mbox{ for some } \freqbin \in [\nrfreqs] \}$. We also use $a_{+} \defeq \max \{a,0\}$. 

\vspace*{-3mm}
\section{Problem Formulation}\label{sec:problem_formulation}
\vspace*{-2mm}

Consider a $\coefflen$-dimensional real-valued zero-mean stationary time series 
$\mathbf{x}[\timeidx]=(x_1[\timeidx],\ldots,x_p[\timeidx])^T$, for $\timeidx \in \mathbb{Z}$.
We assume its autocorrelation function (ACF) $\ACF [\lagidx] \defeq \Exp \left\{ \vec{x}[\lagidx] \vec{x}^{T}[0] \right\}$ to be absolutely summable ACF, i.e., where $\sum_{\lagvar=-\infty}^{\infty} \| \ACF[\lagvar] \| < \infty$, 
such that we can define the spectral density matrix (SDM) $\SDM(\theta)$ via a Fourier transform:
\vspace*{-2mm}
\begin{equation}
\label{equ_def_SDM_FT_ACF}
\SDM (\theta) \defeq \sum_{\lagidx=-\infty}^{\infty} \ACF [\lagidx] \exp(-j2\pi
\lagidx \theta) \,, \quad \theta \in [0,1) \,.
\vspace*{-2mm}
\end{equation}
We require the eigenvalues of the SDM to be uniformly bounded as
\vspace*{-2mm}
\begin{equation}
\label{equ_uniform_boundedness_eigvals}
L \leq \lambda_{i} \left(\SDM(\theta)\right) \leq U,
\vspace*{-2mm}
\end{equation}
where, without loss of generality, we will assume $L=1$ in what follows. The upper bound in \eqref{equ_uniform_boundedness_eigvals} is 
valid if the ACF is summable; the lower bound ensures certain Markov properties of the CIG \cite{LauritzenGM,PHDEichler,Dahlhaus2000}. 

Our approach is based on the assumption that the SDM is a smooth function. This smoothness will be quantified via certain ACF moments (cf.\ \eqref{equ_def_SDM_FT_ACF})
\vspace*{-1mm}
\begin{equation}
\label{equ_def_ACF_moment} 
\mu_{x}^{(h)} \defeq \sum_{\lagvar=-\infty}^{\infty} \| \ACF[\lagvar] \|_{\infty} h[\lagvar]. 
\vspace*{-2mm}
\end{equation} 
Here, $h[\lagvar]$ denotes a weight function which typically increases with $|\lagvar|$. 
For a process with sufficiently small moment $\mu_{x}^{(h)}$, thereby enforcing smoothness of the SDM, we are allowed to base our considerations on a 
discretized version of the SDM, given by $\SDM[\freqbin]=\SDM(\theta_{\freqbin})$, with $\theta_{\freqbin} \!\defeq\! (\freqbin\!-\!1)/\nrfreqs$, for $\freqbin\in[\nrfreqs]$. 
The number $\nrfreqs$ of sampling points is a design parameter which has to be chosen suitably large, compared to the ACF moment $\mu_{x}^{(h)}$ (cf. \cite[Lemma 2.1]{JuHeck2014}).

The CIG of a process $\vec{x}[\timeidx]$ is a simple undirected graph 
$\set{G}=(\set{V},\set{E})$ with node set $\set{V} \!= \![p] $. Each node $r \in \set{V}$ represents 
a single scalar component process $x_{r}[\timeidx]$. An edge between nodes $r$ and $r'$ is absent, i.e., $(r,r') \notin \set{E}$, if 
the component processes $x_{r}[\timeidx]$ and $x_{r'}[\timeidx]$ are conditionally independent given all remaining component processes \cite{Dahlhaus2000}.

If the process $\vx[\timeidx]$ is Gaussian, the CIG can be conveniently characterized via the process inverse SDM $\invSDM[\freqbin] \defeq \SDM^{-1}[\freqbin]$. More specifically, 
it can be shown that, for sufficiently small $\mu_{x}^{(h)}$ with $h[\lagvar] = |\lagvar|$,  \cite{Dahlhaus2000,JuHeck2014} 
\vspace*{-1mm}
\begin{equation}
\label{equ_charac_edge_set_inv_SDM}
(i,j) \in \set{E} \Longleftrightarrow (i,j) \in \gsupp(\invSDM[\cdot]). 
\vspace*{-1mm}
\end{equation}  
Thus, the edge set $\set{E}$ of the CIG is determined by the generalized support of the inverse SDM $\invSDM[\freqbin]$, for $\freqbin \in [\nrfreqs]$.  

Our goal is to robustly estimate the CIG from a finite length observation, incurring unavoidable estimation errors. Therefore, we have to require that, in addition to \eqref{equ_charac_edge_set_inv_SDM}, the non-zero off-diagonal entries of $\invSDM[\freqbin]$ are sufficiently large. To this end, we define the process (un-normalized) minimum global partial spectral coherence as
\vspace*{-1mm}
\begin{equation} 
\label{equ_def_rho_x}
\rho_{x} \defeq\min_{(i,j) \in \set{E}} \big\| K_{i,j}[\cdot] \big\|.
\vspace*{-1mm}
\end{equation} 
For the analysis of our GMS scheme we require
\vspace*{-2mm}
\begin{assumption}
\label{apt_rho_min}
We have $\rho_{x} \geq \rho_{\emph{min}}$ for a known $\rho_{\emph{min}}>0$. 
\vspace*{-2mm}
\end{assumption}
\vspace*{-4mm}
Our approach to GMS in the high-dimensional regime exploits a specific problem structure induced by the assumption that the true CIG is sparse. 
\vspace*{-2mm}
\begin{assumption}  
The CIG of the observed process $\vx[\timeidx]$ is sparse such that $| \set{E}| \leq \sparsity$ for some small $\sparsity \ll  \coefflen(\coefflen-1)/2$.
\end{assumption}
\vspace*{-2mm}
The performance analysis of the proposed GMS algorithm requires to quantify the conditioning of SDM sub-matrices. 
In particular, we will use the following assumption which is a natural extension of the (multitask) compatibility condition, originally introduced 
in \cite{BuhlGeerBook} to analyze LASSO for the ordinary sparse linear (multitask) model. 
\vspace*{-2mm}
\begin{assumption}
\label{equ_asspt_compat_cond}
Given a process $\vx[\timeidx]$ whose CIG contains no more than $\sparsity$ edges indexed by  $\mathcal{S}\subseteq [\coefflen]\times[\coefflen]$, we assume that there exists a positive 
constant $\phi$ such that 
\vspace*{-2mm}
\begin{equation}
\label{label_equ_bound_compat_condition}
\frac{1}{\nrfreqs} \hspace*{-1mm}\sum_{\freqbin \in [\nrfreqs]} \hspace*{-2mm}\vectorize\{ \optvar[\freqbin] \}^{H} \! \big( \bar{\SDM}[\freqbin] \otimes \SDM[\freqbin] \big) \!\vectorize \{ \optvar[\freqbin] \} \!\geq\! \frac{\phi}{\sparsity} \| \optvar_{\mathcal{S}} \|^{2}_{1}
\vspace*{-2mm}
\end{equation} 
holds for all $\optvar[\cdot] \in \mathcal{H}_{\coefflen}^{[\nrfreqs]}$ with $\| \optvar_{\mathcal{S}^{c}} \|_{1} \leq 3 \|  \optvar_{\mathcal{S}}\|_{1}$. 
\end{assumption}
The constant $\phi$ in \eqref{label_equ_bound_compat_condition} is essentially a lower bound on the eigenvalues of small sub-matrices of the SDM. 
As such, the Assumption \ref{equ_asspt_compat_cond} is closely related to the concept of the restricted isometry property (RIP) \cite{GeerBuhlConditions}. 

It can be verified easily that Assumption \ref{equ_asspt_compat_cond} is always valid with $\phi \!=\! L$ for a process satisfying \eqref{equ_uniform_boundedness_eigvals}. However, 
for processes having a sparse CIG, we typically expect $\phi \gg L$.

\vspace*{-2mm}
\section{Graphical LASSO for Time Series}\label{sec:graphical_lasso}
\vspace*{-1mm}

The \textit{graphical least absolute shrinkage and selection operator} (gLASSO) \cite{FriedHastieTibsh2008,BuhlGeerBook,WittenFriedSimon2011,RavWainRaskYu2011} is
an algorithm for estimating the inverse covariance matrix $\invSDM \defeq \mathbf{C}^{-1}$ of a high-dimensional Gaussian random vector $\mathbf{x} \sim \mathcal{N}(\mathbf{0},\mathbf{C})$ based 
on i.i.d.\ samples. 
In particular, gLASSO is based on optimizing a $\ell_1$-penalized log-likelihood function and can therefore be interpreted as regularized maximum likelihood estimation.

\vspace*{-5mm}
\subsection{Extending gLASSO to stationary time series}
\vspace*{-1mm}
A natural extension of gLASSO to the case of stationary Gaussian time series is to replace the objective 
function for the i.i.d.\ case with the corresponding penalized log-likelihood function for a stationary Gaussian process. 
However, since the exact likelihood lacks a simple closed-form expression, we will use the popular ``Whittle-approximation'' \cite{BachJordan04,whittle53}, to arrive at 
the following gLASSO estimator for general stationary time series: 
\vspace*{-1mm}
\begin{equation}
\label{eq:graphical_lasso_discr_time_series}
\GLASSO[\cdot] \defeq \argmin_{\optvar[\cdot] \in \mathcal{C}} \!-A\{\optvar\} + \langle \SDMhat[\cdot],
\optvar[\cdot] \rangle + \lambda \| \optvar[\cdot]\|_1 
\vspace*{-1mm}
\end{equation}
with $A\{\optvar\} \defeq  (1/\nrfreqs) \sum_{\freqbin \in [\nrfreqs]} \log| \optvar[\freqbin]|$ and 
\begin{equation}
\label{equ_def_constraint_set_C}
\mathcal{C} \!\defeq\! \{ \optvar[\cdot] \in \mathcal{H}_{\coefflen}^{[\nrfreqs]}:  \mathbf{0} \prec \optvar[\freqbin] \preceq  \mathbf{I} \mbox{ for all } \freqbin \in [\nrfreqs]\}.
\end{equation} 
The constraint set is reasonable since (i) the function $A\{\optvar\}$ is only finite if $\optvar[\freqbin] \!\succ\! \mathbf{0}$ and (ii) the true 
inverse SDM satisfies $\invSDM[\freqbin] \preceq \mathbf{I}$, for all $\freqbin \in [\nrfreqs]$, due to \eqref{equ_uniform_boundedness_eigvals} (with $L\!=\!1$). 

The formulation \eqref{eq:graphical_lasso_discr_time_series} involves an estimator $\ESDM[\freqbin]$ of the SDM values $\SDM(\theta_{\freqbin})$, for $\freqbin \in [\nrfreqs]$. 
While in principle any reasonable estimator could be used, we will restrict the choice to a multivariate Blackman-Tukey (BT) estimator \cite{stoi97}
\vspace*{-3mm}
\begin{equation}
\label{equ_def_BT_SDM_estimate}
\ESDM[\freqbin] = \sum_{\lagidx=-\samplesize+1}^{\samplesize-1}
w[\lagidx] \widehat{\mat{R}}[\lagidx] \exp(-j 2\pi \lagidx \theta_{\freqbin})
\vspace*{-2mm}
\end{equation}
with the standard biased autocorrelation estimate
$\widehat{\mat{R}}[\lagidx] = (1/\samplesize) \sum_{\timeidx=\lagidx+1}^{\samplesize} \vec{x}[\timeidx] \vec{x}^{T}[\timeidx-\lagidx]$ for $\lagidx = 0,\ldots,\samplesize-1$. 
Enforcing the symmetry $\widehat{\mat{R}}[-\lagidx]
=\widehat{\mat{R}}^H[\lagidx]$, we can obtain the ACF
estimate for $\lagidx = -\samplesize+1,\ldots,-1$.
The window function $w[\lagidx]$ in \eqref{equ_def_BT_SDM_estimate} is a design parameter, 
which can be chosen freely as long as it yields a psd estimate $\ESDM[\freqbin]$. A sufficient condition such that $\ESDM[\freqbin]$ is 
guaranteed to be psd is non-negativeness of the Fourier transform $W(\theta) \!\defeq\! \sum_{\lagidx} w[\lagidx] \exp(- j 2 \pi \theta \lagidx)$Ê\cite[Sec.\ 2.5.2]{stoi97}. 

The existence of a minimizer in \eqref{eq:graphical_lasso_discr_time_series} is guaranteed for any choice of $\lambda\geq 0$ as the optimization problem \eqref{eq:graphical_lasso_discr_time_series} is equivalent 
to the unconstrained problem $\min_{\optvar[\cdot]}  \!-A\{\optvar\} + \langle \ESDM[\cdot],
\optvar[\cdot] \rangle + \lambda \| \optvar[\cdot]\|_1 + I_{\mathcal{C}}(\optvar[\cdot])$, where  $I_{\mathcal{C}}(\optvar[\cdot])$ is the 
indicator function of the constraint set $\mathcal{C}$. Existence of a minimizer of this equivalent problem is guaranteed by \cite[Theorem 27.2]{RockafellarBook}: The objective function is a closed proper convex function and is finite only on the bounded set $\mathcal{C}$, which trivially implies that 
the objective function has no direction of recession \cite{RockafellarBook}. 

We will present in Sec.\ \ref{sec_perfomance_analysis} a specific choice for $\lambda$ in \eqref{eq:graphical_lasso_discr_time_series}, which ensures that the gLASSO estimator $\GLASSO[\cdot]$ is accurate, i.e., the 
estimation error $\GLASSOERR[\cdot] \defeq \GLASSO[\cdot] -  \invSDM[\cdot]$ is small. 
Based on the gLASSO \eqref{eq:graphical_lasso_discr_time_series}, an estimate for the edge set of the CIG may then be obtained by 
thresholding: 
\vspace*{-1mm}
\begin{equation}
\label{equ_def_estimated_edge_set}
\widehat{\set{E}}(\eta) \defeq \{ (i,j):  \big\| \widehat{K}_{i,j}[\cdot]Ê\big\| \geq \eta \}.
\vspace*{-1mm}
\end{equation}
Obviously, under Asspt. \ref{apt_rho_min}, if
\vspace*{-1mm}
\begin{equation} 
\label{equ_condition_norm_error_below_rho_min_2}
\| \GLASSOERR \|_{1} < \rho_{\text{min}}/2, 
\vspace*{-1mm}
\end{equation}
we have $\widehat{\set{E}}( \rho_{\text{min}}/2) = \set{E}$, i.e., the CIG is recovered perfectly. 

\vspace*{-3mm}
\subsection{ADMM Implementation}

An efficient numerical method for solving convex optimization problems of the type \eqref{eq:graphical_lasso_discr_time_series} is the \textit{alternating direction method of multipliers}
(ADMM). 
Defining the augmented Lagrangian \cite{DistrOptStatistLearningADMM} of the problem \eqref{eq:graphical_lasso_discr_time_series} as 
\vspace*{-1mm}
\begin{align} 
L_{\rho}(\optvar[\cdot],\secondvar[\cdot],\dualvar[\cdot]) &  \defeq -A\{\optvar\} + \langle \ESDM[\cdot] , \optvar[\cdot] \rangle+ \lambda \| \secondvar \|_{1}  \nonumber\\[-0mm]Ê
& \hspace*{5mm} + (\rho/2) \|\optvar[\cdot]-\secondvar[\cdot]  + \dualvar[\cdot] \|_{\rm{F}}^{2}, \nonumber \\[-5mm]
\nonumber
\end{align}
the (scaled) ADMM method iterates, starting with an arbitrary initial guess for $\optvar^{0}[\cdot]$, $\secondvar^{0}[\cdot]$ and $\dualvar^{0}[\cdot]$, the following update rules 
\begin{align}
\optvar^{k+1}[\cdot] & = \argmin_{\optvar[\cdot] \in \mathcal{C}} L_{\rho}(\optvar[\cdot], \secondvar^{k}[\cdot],\dualvar^{k}[\cdot]) \label{equ_def_primal_update_ADMM}\\ 
\secondvar^{k+1}[\cdot] & = \argmin_{\secondvar[\cdot] \in \mathcal{H}_{\coefflen}^{[\nrfreqs]}} L_{\rho}(\optvar^{k+1}[\cdot], \secondvar[\cdot],\dualvar^{k}[\cdot]) \label{equ_def_second_update_ADMM}\ \\ 
\dualvar^{k+1}[\cdot] &= \dualvar^{k}[\cdot] + (\optvar^{k+1}[\cdot] - \secondvar^{k+1}[\cdot]) \label{equ_def_dual_update_ADMM}. 
\end{align} 
It can be shown (cf.\ \cite[Sec. 3.2]{DistrOptStatistLearningADMM}) that for any $\rho>0$, the iterates $\optvar^{k}[\cdot]$ converge to a solution of \eqref{eq:graphical_lasso_discr_time_series}. Thus, while the precise choice for $\rho$ has some influence on the convergence speed of ADMM \cite[Sec. 3.4.1]{DistrOptStatistLearningADMM}, it is not very crucial. We used $\rho\!=\!100$ in all of our experiments (cf.\ Section \ref{sec:results}).

Closed-forms for updates \eqref{equ_def_primal_update_ADMM} and \eqref{equ_def_second_update_ADMM} are stated in 
\vspace*{-1mm}
\begin{proposition}
\label{prop_ADMM_rules_gLASSO}
Let us denote the eigenvalue decomposition of the matrix $\ESDM[\freqbin]\!+\!\rho(\dualvar^{k}[\freqbin] \!-\!  \secondvar^{k}[\freqbin])$ by $\mathbf{V}_{\freqbin} \mathbf{D}_{\freqbin} \mathbf{V}_{\freqbin}^{H}$ 
with the diagonal matrix $\mathbf{D}_{\freqbin}$ composed of the eigenvalues $d_{\freqbin,i}$, sorted non-increasingly. 
Then, the ADMM update rule \eqref{equ_def_primal_update_ADMM} is accomplished by setting, separately for each $\freqbin \in [\nrfreqs]$, 
\vspace*{-2mm}
\begin{equation} 
\label{equ_update_rule_close_form_optvar}
\optvar^{k+1}[\freqbin] = \mathbf{V}_{\freqbin} \widetilde{\mathbf{D}}_{\freqbin} \mathbf{V}_{\freqbin}^{H}  
\vspace*{-2mm}
\end{equation} 
with the diagonal matrix $ \widetilde{\mathbf{D}}_{\freqbin}$ whose $i$th diagonal element is given by 
$\tilde{d}_{\freqbin,i} = \min\big\{ (1/(2\rho)) \big[\!-\!d_{\freqbin,i} \!+\! \sqrt{d_{\freqbin,i}^{2} \!+\! 4 \rho}\big], 1 \big\}$. 

If we define the block-thresholding operator $\mathbf{W}[\cdot] \!=\! \mathcal{S}_{\kappa}(\mathbf{Y}[\cdot])$ via 
$W_{i,j}[\freqbin] \!\defeq\! \big(1\!-\! \kappa/ \| Y_{i,j}[\cdot] \| \big)_{+} Y_{i,j}[\freqbin]$, 
the update rule \eqref{equ_def_second_update_ADMM} results in 
\vspace*{-2mm}
\begin{equation}
\label{equ_explicit_updata_z_soft_thr}
\secondvar^{k+1}[\cdot] = \mathcal{S}_{\lambda/\rho}\big( \optvar^{k+1}[\cdot] \!+\! \dualvar^{k}[\cdot] \big).  
\vspace*{-2mm}
\end{equation} 
\end{proposition}
\begin{proof}
Since the minimization problem \eqref{equ_def_second_update_ADMM} is equivalent to the ADMM update for a group LASSO problem \cite[Sec. 6.4.2]{DistrOptStatistLearningADMM}, the explicit form \eqref{equ_explicit_updata_z_soft_thr} follows from the derivation in \cite[Sec. 6.4.2]{DistrOptStatistLearningADMM}. 

Note that the optimization problem \eqref{equ_def_primal_update_ADMM} splits into $\nrfreqs$ separate subproblems, one for each $\freqbin \!\in\! [\nrfreqs]$. 
The subproblem for a specific frequency bin $\freqbin$ is (omitting the index $\freqbin$) 
\vspace*{-1mm}
\begin{equation}
\label{equ_sub_problem_fixed_freq_x_updat}
\min_{\mathbf{0} \prec \mathbf{X} \preceq  \mathbf{I}} - \log |\optvar|+ \langle \optvar,\ESDM+\rho(\dualvar^{k}-\secondvar^{k})  \rangle +  (\rho/2)\| \optvar \|_{\rm{F}}^{2}.
\vspace*{-1mm}
\end{equation}
Let us denote the non-increasing eigenvalues of the Hermitian matrices $\mathbf{X}$ and $\mathbf{Y} \!\defeq\! \ESDM \!+\! \rho(\dualvar^{k}\!-\!\secondvar^{k})$ by $x_{i}$ and $d_{i}$, for $i \!\in\! [\coefflen]$, respectively. 

According to \cite[Lemma II.1]{Lasserre95}, we have the trace inequality  $\langle \optvar,\mathbf{Y}  \rangle \!\geq\! \sum_{i \in [\coefflen]} x_{i} d_{\coefflen-i-1}$ with equality 
if $\mathbf{X}$ is of the form $\mathbf{X} \!=\! \mathbf{V} \mbox{diag} \{ d_{\coefflen-i-1} \}\mathbf{V}^{H}$ with a unitary matrix $\mathbf{V}$ containing eigenvectors of $\mathbf{Y}$. 
Based on this trace inequality, a lower bound on the minimum in \eqref{equ_sub_problem_fixed_freq_x_updat} 
is given by 
\vspace*{-1mm}
\begin{equation} 
\label{equ_equiv_problem_eigvals_sorted_removed}
\min_{ 0< x_{i} \leq 1} \sum_{iÊ\in [\coefflen]} -\log x_{i} + x_{i} d_{\coefflen-i+1}  + (\rho/2) x_{i}^{2}.
\vspace*{-1mm}
\end{equation} 
The minimum in \eqref{equ_equiv_problem_eigvals_sorted_removed} is achieved by the choice $\tilde{x}_{i} \!=\! h(d_{\coefflen-i+1})$ with $h(z) \!\defeq\! \min\{(- z + \sqrt{z^2 + 4\rho})/2\rho,1 \}$. For the choice $\mathbf{X} \!=\! \mathbf{V} \diag\{\tilde{x}_{i}\} \mathbf{V}^{H}$ (which is \eqref{equ_update_rule_close_form_optvar}), the objective function in \eqref{equ_sub_problem_fixed_freq_x_updat} achieves the lower bound \eqref{equ_equiv_problem_eigvals_sorted_removed}, certifying optimality.
\vspace*{-3mm}
\end{proof} 
We summarize our GMS method in 
\vspace*{-2mm}
\begin{algorithm}
\label{algo_GMS_GLASSO}
Given samples $\vx[1],...,\vx[N]$, parameters $T$, $\nrfreqs$, $\eta$, $\lambda$ and window function $w[\lagvar]$ perform the steps:

{\bf Step 1: }
For each $\freqbin \in [\nrfreqs]$, compute the SDM estimate $\ESDM[\freqbin]$ according 
to \eqref{equ_def_BT_SDM_estimate}. 

{\bf Step 2:}
Compute gLASSO $\GLASSO[\cdot]$ (cf.\ \eqref{eq:graphical_lasso_discr_time_series}) by iterating \eqref{equ_update_rule_close_form_optvar}, \eqref{equ_explicit_updata_z_soft_thr} and \eqref{equ_def_dual_update_ADMM} for a fixed number $T$. 

{\bf Step 3:}
Estimate the edge set via $\widehat{\set{E}}(\eta)$ (cf.\ \eqref{equ_def_estimated_edge_set}).  

\end{algorithm}

\vspace*{-5mm}
\subsection{Performance Analysis}
\label{sec_perfomance_analysis} 
\vspace*{-2mm}

Let us for simplicity assume that the ADMM iterates $\optvar^{k}[\cdot]$ converged perfectly to the gLASSO estimate $\GLASSO[\cdot]$ given by \eqref{eq:graphical_lasso_discr_time_series}. 
Under Asspt. \ref{apt_rho_min}, a sufficient condition for our GMS method to succeed is \eqref{equ_condition_norm_error_below_rho_min_2}.  

We will now derive an upper bound on the probability that \eqref{equ_condition_norm_error_below_rho_min_2} fails to hold. This will be accomplished 
by (i) showing that the norm $ \| \GLASSOERR \|_{1}$ can be bounded in terms of the SDM estimation error $\mathbf{E}[\freqbin] \!\defeq\! \SDM[\freqbin] \!-\!\ESDM[\freqbin]$ and (ii) controlling 
the probability that the error $\mathbf{E}[\freqbin]$ is sufficiently small. 

\emph{Upper bounding $\| \GLASSOERR \|_{1}$.}
By definition of $\GLASSO[\cdot]$ (cf. \eqref{eq:graphical_lasso_discr_time_series}), 
\vspace*{-1mm}
\begin{equation}
\label{equ_minimizer_GLASSO_cond1}
\hspace*{-2mm} -A\{\GLASSO\} \!+\! \langle \GLASSOERR[\cdot], \ESDM[\cdot] \rangle \!+\! \lambda (\| \GLASSO \|_{1}\!-\!\| \invSDM\|_{1}) \!\leq\! -A\{\invSDM\}. 
\vspace*{-1mm}
\end{equation}
Combining with  
$\argmin_{\optvar[\cdot]  \in \mathcal{C}}  - A\{\optvar\} \!+\!  \langle \optvar[\cdot], \SDM[\cdot] \rangle \!=\! \invSDM[\cdot]$, 
\vspace*{-2mm}
\begin{align}
\label{equ_minimizer_GLASSO_cond2}
 \lambda \| \GLASSO \|_{1} \leq \lambda \| \invSDM \|_{1} +  \langle  \GLASSOERR[\cdot], \mathbf{E}[\cdot] \rangle.
 \vspace*{-2mm}
\end{align}

\vspace*{-1mm}
Let us, for the moment, assume validity of the condition 
\vspace*{-1mm}
\begin{equation}
\label{equ_cond_max_err_lambda_2}
E \defeq \max_{\freqbin \in [\nrfreqs]} \| \mathbf{E}[\freqbin] \|_{\infty} \leq \lambda/2, 
\vspace*{-2mm}
\end{equation} 
implying $| \langle  \GLASSOERR[\cdot], \mathbf{E}[\cdot] \rangle|  \! \leq \! ( \lambda/2)\|  \GLASSOERR\|_{1}$ and, in turn via \eqref{equ_minimizer_GLASSO_cond2}, 
\vspace*{-1mm}
\begin{equation}
\label{equ_minimizer_GLASSO_sparse_1}
 \lambda \| \GLASSO \|_{1} \leq \lambda \| \invSDM \|_{1} \!+\! (\lambda/2) \|  \GLASSOERR \|_{1}.
 \vspace*{-1mm}
 \end{equation} 
 Applying the (reverse) triangle inequality on both sides, 
 \vspace*{-2mm}
 \begin{equation} 
  \|  \GLASSOERR_{\mathcal{S}^{c}}\|_{1} \stackrel{(a)}{=} \| \GLASSO_{\mathcal{S}^{c}} \|_{1} \leq 3 \|  \GLASSOERR_{\mathcal{S}}\|_{1}, 
 \vspace*{-1mm}
 \end{equation} 
 where $(a)$ is due to $\mathcal{S} = \gsupp (\invSDM[\cdot])$. 
Thus, the estimation error $\GLASSOERR[\cdot]$ tends to be sparse. 
 
As a next step, we rewrite \eqref{equ_minimizer_GLASSO_cond1} as 
\vspace*{-1mm}
 \begin{equation}
 \label{equ_miniizer_gLASSO_bound_norm_1}
 -(A\{\GLASSO\}\!-\!A\{\invSDM\})\!+\!\langle \GLASSOERR[\cdot], \SDM[\cdot] \rangle \leq  (3\lambda/2) \| \GLASSOERR \|_{1}  
 \vspace*{-1mm}
 \end{equation} 
where we used \eqref{equ_cond_max_err_lambda_2}. 
Let us denote the eigenvalues of the psd matrix  ${\bm \Gamma}[\freqbin] \defeq \SDM^{1/2}[\freqbin]\GLASSOERR[\freqbin]\SDM^{1/2}[\freqbin]$ by $\gamma_{\freqbin,i}$. 
We can then reformulate the LHS of \eqref{equ_miniizer_gLASSO_bound_norm_1} as 
\vspace*{-2mm}
\begin{equation}
 \hspace*{-3mm}-(A\{\GLASSO\}\!-\!A\{\invSDM\})\!+\!\langle \GLASSOERR[\cdot], \SDM[\cdot] \rangle 
 \!=\!  \frac{1}{\nrfreqs} \sum_{\freqbin,i} \hspace*{-1mm}-\log(1\!+\!\gamma_{\freqbin,i})\!+\!\gamma_{\freqbin,i}. \nonumber
 \vspace*{-3mm}
\end{equation} 
Since $\GLASSO[\freqbin], \invSDM[\freqbin] \preceq \mathbf{I}$ and $\SDM[\freqbin] \preceq U\mathbf{I}$ (cf. \eqref{equ_def_constraint_set_C}, \eqref{equ_uniform_boundedness_eigvals}), we have $\gamma_{\freqbin,i} \leq 2U$. 
Using $\log(1\!+\!x)\!=\!x\!-\!\frac{x^{2}}{2(1\!+\!\varepsilon x)^{2}}$, with some $\varepsilon \in [0,1]$, we further obtain 
\vspace*{-2mm}
\begin{equation}
 \label{equ_miniizer_gLASSO_bound_norm_3}
 \hspace*{-3mm}-(A\{\GLASSO\}\!-\!A\{\invSDM\})\!+\!\langle \GLASSOERR[\cdot], \SDM[\cdot] \rangle \geq 
\|{\bm \Gamma}[\cdot]  \|_{\rm F}^{2}/(4(2U\!+\!1)^2). \nonumber
\vspace*{-2mm}
\end{equation} 
Applying \cite[Lemma 4.3.1 (b)]{Horn91} to the RHS,   
\vspace*{-1mm}
\begin{align}
\label{equ_miniizer_gLASSO_bound_norm_4} 
 &-(A\{\GLASSO\}\!-\!A\{\invSDM\})\!+\!\langle \GLASSOERR[\cdot], \SDM[\cdot] \rangle \\Ê
 & \hspace*{-3mm} \!\geq\! \frac{1}{4 \nrfreqs (2U\!+\!1)^2} \hspace*{-1mm}\sum_{\freqbin \in [\nrfreqs]} \hspace*{-2mm} \vectorize\{\GLASSOERR[\freqbin]\}^{H} (\bar{\SDM}[\freqbin] \otimes \SDM[\freqbin])\vectorize\{\GLASSOERR[\freqbin]\}. \nonumber \\[-8mm]
 \nonumber
\end{align} 
Combining \eqref{equ_miniizer_gLASSO_bound_norm_4} with \eqref{equ_miniizer_gLASSO_bound_norm_1} and Asspt. \ref{equ_asspt_compat_cond}, we arrive at 
\vspace*{-2mm}
\begin{equation}
\label{equ_final_upper_bound_glasso_error_norm}
 \| \GLASSOERR \|_{1} \leq 96 (2U\!+\!1)^2 (\sparsity/\phi) \lambda.
 \vspace*{-1mm}
\end{equation} 

\vspace*{-1mm}
\emph{Controlling the SDM estimation error.}
It remains to control the probability that condition \eqref{equ_cond_max_err_lambda_2} is valid, i.e., the SDM estimation error $\mathbf{E}[\freqbin]$ incurred by 
the BT estimator \eqref{equ_def_BT_SDM_estimate} is sufficiently small. 
According to \cite{CSGraphSelJournal}, for any $\nu \in [0,1/2)$,  
\vspace*{-1mm}
\begin{equation}
\label{equ_upper_bound_fail_event}
\PR\{ E \! \geq\! \nu + \mu_{x}^{(h_{1})} \} \leq 2e^{-\frac{(1/32)\samplesize \nu^2}{\| w[\cdot] \|^{2}_{1} U^2}  +2\log (2\coefflen \samplesize)  }.
\vspace*{-1mm}
\end{equation}
where $\mu_{x}^{(h_{1})}$ is the ACF moment \eqref{equ_def_ACF_moment} with $h_{1}[\lagvar]\!\defeq\! |1\!-\!w[\lagvar](1\!-\!|\lagvar|/ \samplesize)|$  for $|\lagvar|\!<\! \samplesize$ and $h_{1}[\lagvar]\!\defeq\!1$ else.


\emph{The main result.} 
Recall that a sufficient condition for $\widehat{\set{E}}(\rho_{\text{min}}/2)$, given by \eqref{equ_def_estimated_edge_set}, to coincide with the true edge set is \eqref{equ_condition_norm_error_below_rho_min_2}. 
Under the condition \eqref{equ_cond_max_err_lambda_2}, implying validity of \eqref{equ_final_upper_bound_glasso_error_norm}, the inequality \eqref{equ_condition_norm_error_below_rho_min_2} will be satisfied if $\lambda$ 
is chosen as 
\vspace*{-1mm}
\begin{equation}
\label{equ_choice_lambda}
\lambda = \phi \rho_{\text{min}}/ (192 \sparsity (2U\!+\!1)^{2}).
\vspace*{-1mm}
\end{equation}
Using \eqref{equ_upper_bound_fail_event} to bound the probability for \eqref{equ_cond_max_err_lambda_2} to hold yields
\begin{proposition} 
Consider a stationary Gaussian zero-mean time series $\vx[\timeidx]$ satisfying \eqref{equ_uniform_boundedness_eigvals} and Asspt. \ref{apt_rho_min}-\ref{equ_asspt_compat_cond}.
Then, using the choice \eqref{equ_choice_lambda} in \eqref{eq:graphical_lasso_discr_time_series} and if the conditions 
\vspace*{-2mm}
\begin{align} 
\label{equ_cond_ACF_moment}
\mu_{x}^{(h_{1})} & \leq \phi \rho_{\emph{min}}/ (384 \sparsity (2U\!+\!1)^{2}) \\ 
\label{equ_main_result_suff_cond_sample_size}
\samplesize & \geq  10^8 (2U\!+\!1)^6 \sparsity^2 \phi^{-2} \rho_{\emph{min}}^{-2}   \|w[\cdot]\|_{1}^{2} \log(4 \samplesize p^2/\delta) 
\vspace*{-2mm}
\end{align} 
are satisfied, we have $\PR \{ \widehat{\set{E}}(\rho_{\emph{min}}/2) \neq \set{E} \} \leq \delta$.
\end{proposition} 
In order to satisfy the condition \eqref{equ_cond_ACF_moment}, the window function $w[\cdot]$ in \eqref{equ_def_BT_SDM_estimate} has to be chosen as the indicator function 
for the effective support of the ACF $\ACF[\lagidx]$. Thus, the factor $  \|w[\cdot]\|_{1}^{2}$ in \eqref{equ_main_result_suff_cond_sample_size} corresponds to a scaling of the sample size with the square of the effective ACF width. 
Moreover, the sufficient condition \eqref{equ_main_result_suff_cond_sample_size} scales inversely with the square of the minimum partial spectral coherence $\rho_{\text{min}}$ 
which agrees with the scaling of the sufficient condition obtained for the neighborhood regression approach in \cite{JuHeck2014}. 

\vspace*{-2mm}
\section{Numerical results}\label{sec:results}
\vspace*{-2mm} 

\begin{figure}
\begin{center}
\psfrag{PD}[c][c][.9]{\uput{4mm}[90]{0}{\hspace{0mm}$P_{d}$}}
\psfrag{FA}[c][c][.85]{\uput{3mm}[270]{0}{\hspace{0mm}$P_{fa}$}}
\psfrag{ROC}[c][c][.9]{\uput{2.5mm}[90]{0}{}}
\psfrag{x_0}[c][c][.9]{\uput{0.3mm}[270]{0}{$0$}}
\psfrag{x_0_1}[c][c][.9]{\uput{0.3mm}[270]{0}{$0.1$}}
\psfrag{x_0_2}[c][c][.9]{\uput{0.3mm}[270]{0}{$0.2$}}
\psfrag{x_0_3}[c][c][.9]{\uput{0.3mm}[270]{0}{$0.3$}}
\psfrag{x_0_4}[c][c][.9]{\uput{0.3mm}[270]{0}{$0.4$}}
\psfrag{x_0_5}[c][c][.9]{\uput{0.3mm}[270]{0}{$0.5$}}
\psfrag{x_0_6}[c][c][.9]{\uput{0.3mm}[270]{0}{$0.6$}}
\psfrag{x_0_7}[c][c][.9]{\uput{0.3mm}[270]{0}{$0.7$}}
\psfrag{x_0_8}[c][c][.9]{\uput{0.3mm}[270]{0}{$0.8$}}
\psfrag{x_0_9}[c][c][.9]{\uput{0.3mm}[270]{0}{$0.9$}}
\psfrag{x_1}[c][c][.9]{\uput{0.3mm}[270]{0}{$1$}}
\psfrag{y_0}[c][c][.9]{\uput{0.3mm}[180]{0}{$0$}}
\psfrag{y_0_1}[c][c][.9]{\uput{0.3mm}[180]{0}{$0.1$}}
\psfrag{y_0_2}[c][c][.9]{\uput{0.3mm}[180]{0}{$0.2$}}
\psfrag{y_0_3}[c][c][.9]{\uput{0.3mm}[180]{0}{$0.3$}}
\psfrag{y_0_4}[c][c][.9]{\uput{0.3mm}[180]{0}{$0.4$}}
\psfrag{y_0_5}[c][c][.9]{\uput{0.3mm}[180]{0}{$0.5$}}
\psfrag{y_0_6}[c][c][.9]{\uput{0.3mm}[180]{0}{$0.6$}}
\psfrag{y_0_7}[c][c][.9]{\uput{0.3mm}[180]{0}{$0.7$}}
\psfrag{y_0_8}[c][c][.9]{\uput{0.3mm}[180]{0}{$0.8$}}
\psfrag{y_0_9}[c][c][.9]{\uput{0.3mm}[180]{0}{$0.9$}}
\psfrag{y_1}[c][c][.9]{\uput{0.3mm}[180]{0}{$1$}}
\psfrag{data1}[c][c][.9]{\uput{0.3mm}[0]{0}{$N\!=\!256$}}
\psfrag{data2}[c][c][.9]{\uput{0.3mm}[0]{0}{$N\!=\!128$}}
\hspace*{-2mm}\includegraphics[width=9cm,height=6cm]{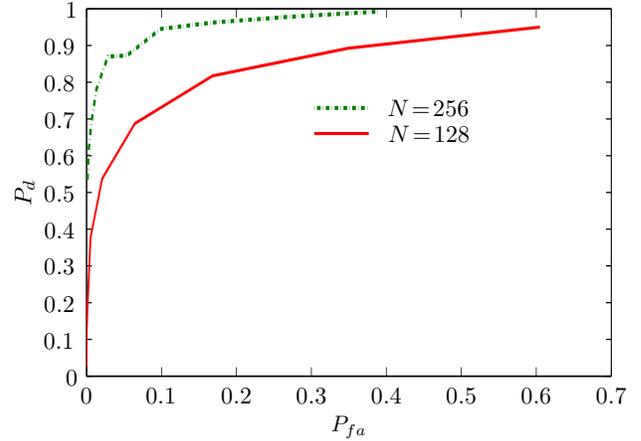}
\end{center}
\vspace*{-3mm}
  \caption{ROC curves of gLASSO based GMS method.} 
\label{fig_ROC}
\vspace*{-3mm}
\end{figure}


We generated 
a Gaussian time series $\mathbf{x}[n]$ of dimension $p\!=\!64$ by applying a finite impulse 
response (FIR) filter $g[n]$ of length $2$ to a zero-mean, stationary, white, Gaussian noise process $\mathbf{e}[n] \sim \mathcal{N}(\mathbf{0},\mathbf{C}_{0})$. 
We choose the covariance matrix $\mathbf{C}_{0}$ such that the resulting CIG $\cig=([p],\set{E})$ is a star graph containing a hub node with $|\set{E}|=4$ neighbors. 
The corresponding precision matrix $\mathbf{K}_{0} = \mathbf{C}_{0}^{-1}$ has main diagonal entries equal to $1/2$ and off diagonal entries equal to $4/(10\sparsity)$.  
The filter coefficients $g[n]$ are such that the magnitude of the associated transfer function is uniformly bounded from above and below by positive constants, thereby ensuring that conditions \eqref{equ_uniform_boundedness_eigvals} and \eqref{equ_charac_edge_set_inv_SDM} are satisfied with $L\!=\!1$, $U\!=\!10$ and $\nrfreqs\!=\!4$.  
Thus, the generated time series satisfies Asspt.\ \ref{apt_rho_min} - \ref{equ_asspt_compat_cond}  with $\rho_{\text{min}}\!=\! 0.1 \sqrt{ \sum_{k\in \mathbb{Z}} r_{g}^{2}[k]}$, $\sparsity\!=\!4$ and $\phi=1$. Here, $r_{g}[m] = \sum_{m \in \mathbb{Z}} g[n+m] g[n]$ denotes the autocorrelation sequence of $g[n]$. 

Based on $\samplesize \in \{128,256\}$ observed samples, we estimated the edge set of the CIG using Algorithm \ref{algo_GMS_GLASSO} with number of ADMM iterations $L=10$, number of frequency points $\nrfreqs=4$ and the window function $w[\lagvar] \!=\! \exp(-\lagvar^2 )$. The gLASSO parameter $\lambda$ (cf. \eqref{eq:graphical_lasso_discr_time_series}) was varied in the range $[1.25,7.5]$. 


In Fig.~\ref{fig_ROC}, we show receiver operating characteristic (ROC) curves with the average fraction of false alarms $P_{f\!a} \!\defeq\! \frac{1}{M} \sum\limits_{i \in [M]} \frac{|\widehat{\set{E}}_{i} \setminus \set{E}| }{p(p-1)/2-|\set{E}|}$ and the average fraction of correct decisions $P_{d} \!\defeq\! \frac{1}{M} \sum\limits_{i \in [M]} \frac{|\widehat{\set{E}}_{i} \cap \set{E}|}{|\set{E}|}$ averaged over $M\!=\!100$ independent simulation runs. Here, $\widehat{\set{E}}_{i}$ denotes the estimate \eqref{equ_def_estimated_edge_set} in the $i$-th simulation run. Each point in the curves correspond to a specific value of the gLASSO parameter $\lambda$. 

%

\renewcommand{\baselinestretch}{1.05}
\bibliographystyle{IEEEbib}

\bibliography{/Users/ajung/work/LitAJ_ITC.bib,/Users/ajung/work/tf-zentral}
\end{document}